\documentclass[review]{elsarticle}

\usepackage{lineno,hyperref}
\modulolinenumbers[5]

\journal{Journal of \LaTeX\ Templates}

\usepackage[utf8]{inputenc} 
\usepackage[T1]{fontenc}    
\usepackage{hyperref}       
\usepackage{url}            
\usepackage{booktabs}       
\usepackage{amsfonts}       
\usepackage{nicefrac}       
\usepackage{microtype}      
\usepackage{amsfonts,amsmath,amssymb,amsthm,bm}
\usepackage{algorithm}
\usepackage{algorithmic}
\usepackage{graphicx}
\usepackage{bibunits}  
\usepackage{color}

\theoremstyle{plain}
\newtheorem{theorem}{Theorem}[section]

\newtheorem{proposition}[theorem]{Proposition}

\newcommand{\R}{\mathbb{R}}                                     
\newcommand{\C}{\mathbb{C}}                                     
\newcommand{\mat}[3]{#1 \left| \begin{array}{@{}l@{}} \scriptstyle #3 \\ \scriptstyle #2 \end{array}  \right.}

\newcommand{\vectorize}{{\rm vec}}
\newcommand{\bi}{\begin{itemize}}
\newcommand{\ei}{\end{itemize}}
\newcommand{\ba}{\begin{array}}
\newcommand{\ea}{\end{array}}

\newcommand{\ee}{\mbox{\boldmath $e$}}

\newcommand{\FF}{\mbox{\boldmath $F$}}

\newcommand{\YY}{\mbox{\boldmath $Y$}}

\newcommand{\UU}{\mbox{\boldmath $U$}}

\newcommand{\SSigma}{\mbox{\boldmath $\Sigma$}}

\newcommand{\VV}{\mbox{\boldmath $V$}}

\begin{document}

\begin{frontmatter}

\title{Dynamic mode decomposition in vector-valued reproducing kernel Hilbert spaces for extracting dynamical structure among observables}





\author[1,2]{Keisuke Fujii\corref{mycorrespondingauthor}} 
\cortext[mycorrespondingauthor]{Corresponding author}
\ead{keisuke.fujii.zh@riken.jp}
\author[3,2]{Yoshinobu Kawahara}

\address[1]{Graduate School of Informatics, Nagoya University, Furo-cho, Chikusa-ku, Nagoya, Aichi, Japan.}
\address[2]{Center for Advanced Intelligence Project, RIKEN, Furuedai, 6-2-3, Suita, Osaka, Japan.} 
\address[3]{Institute of Mathematics for Industry, Kyushu University, 744 Motooka, Nishi-ku, Fukuoka, Japan.}

\begin{abstract} 
Understanding nonlinear dynamical systems (NLDSs) is challenging in a variety of engineering and scientific fields. Dynamic mode decomposition (DMD), which is a numerical algorithm for the spectral analysis of Koopman operators, has been attracting attention as a way of obtaining global modal descriptions of NLDSs without requiring explicit prior knowledge. However, since existing DMD algorithms are in principle formulated based on the concatenation of scalar observables, it is not directly applicable to data with dependent structures among observables, which take, for example, the form of a sequence of graphs. In this paper, we formulate Koopman spectral analysis for NLDSs with structures among observables and propose an estimation algorithm for this problem. This method can extract and visualize the underlying low-dimensional global dynamics of NLDSs with structures among observables from data, which can be useful in understanding the underlying dynamics of such NLDSs. To this end, we first formulate the problem of estimating spectra of the Koopman operator defined in vector-valued reproducing kernel Hilbert spaces, and then develop an estimation procedure for this problem by reformulating tensor-based DMD. As a special case of our method, we propose the method named as Graph DMD, which is a numerical algorithm for Koopman spectral analysis of graph dynamical systems, using a sequence of adjacency matrices. We investigate the empirical performance of our method by using synthetic and real-world data.
\end{abstract}


\end{frontmatter}


\section{Introduction}
\label{sec:introduction}
Understanding nonlinear dynamical systems (NLDSs) or complex phenomena is a fundamental problem in various scientific and industrial fields. 
Complex systems are broadly defined as systems that comprise non-linearly interacting components~\citep{Boccaletti06}, in fields such as sociology, epidemiology, neuroscience, and physics (e.g., \cite{Centola07,Bullmore09}).
As a method of obtaining a global modal description of NLDSs, operator-theoretic approaches have attracted attention such as in applied mathematics, physics and machine learning. 
One of the approaches is based on the composition operator (usually referred to as the Koopman operator~\citep{Koopman31,Mezic05}), which defines the time evolution of observation functions in a function space.
A strength of this approach is that the spectral analysis of the operator can decompose the global property of NLDSs, because the analysis of NLDSs can be lifted to a linear but infinite dimensional regime. 
This approach can directly obtain dynamical structures such as frequency with delay/growth rate and the spatial coherences corresponding to the temporal information.
Among several estimation methods, one of the most popular algorithms for spectral analysis of the Koopman operator is dynamic mode decomposition (DMD)~\citep{Rowley09,Schmid10}, of which advantage is to extract such a modal description of NLDSs from data, unlike other unsupervised dimensionality reduction methods such as principal component analysis (PCA) for static data.   
DMD has been successfully applied in many real-world problems, such as image processing, neuroscience, and system control (e.g.,~\citep{Kutz16b,Brunton16a}). 
In a machine learning community, several algorithmic improvements have been accomplished by such as a formulation with reproducing kernels and in a Bayesian framework (e.g., \citep{Kawahara16,Takeishi17,Takeishi17c}). 
However, since conventional Koopman spectral analysis and DMDs are in principle formulated based on the concatenation of scalar observables, 
it is not directly applicable to data with dependent structures among observables, which take, for example, the form of a sequence of graphs.

The motivation of this paper is to understand NLDSs with dynamical structures among observables by extracting the low-dimensional global dynamics among observables.
To this end, we develop a formulation of Koopman spectral analysis of NLDSs with structures among observables and propose an estimation algorithm for this problem. 
We first suppose that a sequence of matrices representing the dependency among observables (such as adjacency matrices of graphs) are observed as realizations of structures representing the relation of vector-valued observation function. 
Then, we formulate the problem of estimating the spectra of Koopman operators defined in reproducing kernel Hilbert spaces (RKHSs) endowed with kernels for vector-valued functions, called vector-valued RKHSs (vvRKHSs). 
Recently, there has been an increasing interest in kernels for vvRKHSs dealing with such as classification or regression problem with multiple outputs (e.g., \cite{Alvarez12,Micchelli05} and for the details, see Section~\ref{sec:related}). %
Thus, advantage or contribution of our method is that it can extract and visualize the dynamical structures among observables by incorporating the structure among variables in the vector-valued observation function into the DMD algorithm, which can be useful in understanding the fundamental dynamics behind spatiotemporal data with dependent structures.
Second, we develop an estimation procedure from data by reformulating Tensor-based DMD (TDMD), which can compute DMD from tensor time-series data~\citep{Klus18} without breaking tensor data structure (e.g., a sequence of adjacency matrices). 
We propose a more directly and stably computable TDMD than the previous algorithm.

Furthermore, as a special case of our method, we propose the method named as \textit{Graph DMD}, which is a numerical algorithm for Koopman spectral analysis of graph dynamical systems (GDSs).
GDSs are defined as spatially distributed units that are dynamically coupled according to the structure of a graph~\citep{Cliff16,Mortveit01}.
In mathematics, GDSs have been broadly studied such as in cellular automata~\citep{Mortveit07} and coupled NLDSs~\citep{Wu05}.
Meanwhile, for graph sequence data, researchers have basically computed the graph (spatial) properties in each temporal snapshot (e.g., \cite{Centola07,Bullmore09}) or in a sliding window (e.g., \cite{Ide04}) of the sequence data (for the details, see Section~\ref{sec:related}).
However, these approaches would be difficult to extract the dynamical information directly from graph sequence data.
We consider that our approach will solve this problem to understand the underlying global dynamics of GDSs.

Finally, we investigate the performance of our method with application to several synthetic and real-world datasets, including multi-agent simulation and sharing-bike data. These have the structures among observables, which represent, for example, the relation between agents (such as distance) and the traffic volume between locations, respectively.
 
The remainder of this paper is organized as follows.
First, in Section~\ref{sec:koopman}, we briefly review the background of Koopman spectral analysis and DMD.
Next, we describe the formulation in vvRKHS in Section~\ref{sec:formulation}, and reformulate DMD in a tensor form to estimate from data in Section~\ref{sec:tensorDMD}.
In Section \ref{sec:gdmd}, we propose Graph DMD for the analysis of GDSs.
In Section~\ref{sec:related}, we describe related work.
Finally, we show some experimental results using synthetic and real-world data in Section~\ref{sec:example}, and conclude this paper in Section~\ref{sec:conclusion}.
 
\section{Koopman Spectral Analysis and DMD}
\label{sec:koopman}
Here, we first briefly review Koopman spectral analysis, which is the underlying theory for DMD, and then describe the basic DMD procedure.
First, we consider a NLDS: 
$\bm{x}_{t+1} = \bm{f}(\bm{x}_t)$, where $\bm{x}_t$ is the state vector in the state space $\mathcal{M} \subset \R^{p}$ with time index $t\in\mathbb{T}:=\mathbb{N}_0$ and $\bm{f}\colon\cal{M}\to\cal{M}$ is a (typically, nonlinear) state-transition function.
{\em The Koopman operator}, which we denote by ${\cal K}$, is a linear operator acting on a scalar observation function $g\colon {\cal M}\to \C$ defined by
\begin{equation}\label{eq:koopman}
{\cal{K}}g=g\circ\bm{f},
\end{equation}
where $g\circ\bm{f}$ denotes the composition of $g$ with \textbf{\textit{f}} \citep{Koopman31}.
That is, it maps $g$ to the new function $g\circ\bm{f}$.
We assume that ${\cal K}$ has only discrete spectra.
Then, it generally performs an eigenvalue decomposition:
$\mathcal{K}{\varphi}_{j}(\bm{x})={\lambda}_{j}{\varphi}_{j}(\bm{x})$,
where $\lambda_{j}\in\mathbb{C}$ is the \textit{j}-th eigenvalue 
(called \textit{the Koopman eigenvalue}) and $\varphi_{j}$ is 
the corresponding eigenfunction (called \textit{the Koopman 
eigenfunction}).
We denote the concatenation of scalar functions as $\bm{g} := [g_{1},\ldots, g_{m}]^{\mathrm{T}}$.
If each $g_{j}$ lies 
within the space spanned by the eigenfunction $\varphi_{j}$, we 
can expand the vector-valued $\bm{g} \colon \mathcal M\to \C^m$ in terms of these 
eigenfunctions as $\mbox{\boldmath $g$}(\bm{x})=\sum_{j=1}^{\infty }{\varphi_{j}(\bm{x})\bm{\psi}_{j}}$, 
where $\bm{\psi}_{j}$ is a set of vector 
coefficients called \textit{the Koopman modes}. Through the iterative applications of $\mathcal{K}$, the following equation is obtained:
\begin{flalign}\label{eq:decomposition}
\bm{g}(\bm{x}_t)
=(\bm{g}\circ \underbrace{\bm{f}\circ\cdots\circ\bm{f}}_t)\left(\bm{x}_0\right)
=\sum_{j=1}^{\infty}{\lambda}_j^t{\varphi}_j\left(\bm{x}_0\right)\bm{\psi}_j.
\end{flalign}
Therefore, $\lambda_{j}$ characterizes the time evolution of the 
corresponding Koopman mode $\bm{\psi}_j$, 
i.e., the phase of $\lambda_{j}$ determines its frequency and 
the magnitude determines the growth rate of its dynamics.

Among several possible methods to compute the above modal decomposition from data, DMD \citep{Rowley09,Schmid10} is the most popular algorithm, which estimates an approximation of the decomposition in Eq. \eqref{eq:decomposition}. 
Consider a finite-length observation sequence $\bm{y}_0,\bm{y}_1,\ldots ,\bm{y}_{\tau}$ (${\in \C}^{n}$), where $\bm{y}_{ }:= {\bm{g}}(\bm{x}_t)$. 
Let $\bm{X}= [\bm{y}_0,\bm{y}_1 ,\ldots , \bm{y}_{\tau-1}]$ and $\bm{Y}= [\bm{y}_1,\bm{y}_2,\ldots , \bm{y}_{\tau}]$. 
Then, DMD basically approximates it by calculating the eigendecomposition of 
matrix $\bm{F}=\bm{Y}\bm{X}^{\dagger}$, where $\bm{X}^{\dagger}$ is the pseudo-inverse of $\bm{X}$. The matrix $\bm{F}$ may be intractable to analyze directly when the dimension is large. 
Therefore, in the popular implementation of DMD such as exact DMD \citep{Tu14}, a rank-reduced representation $\hat{\FF}$ based on singular-value decomposition (SVD) is applied.
That is, $\bm{X}\approx {\UU}{\SSigma}{\VV}^*$ and 
$\hat{\FF} =\UU^*\FF\UU= \UU^* \YY \VV \SSigma^{(-1)}$, where $^*$ is the conjugate transpose.
Thereafter, we perform eigendecomposition of $\hat{\FF}$ to obtain the set of the eigenvalues $\lambda_j$ and eigenvectors $\bm{w}_j$.
Then, we estimate the Koopman modes in Eq. (\ref{eq:decomposition}):
$\bm{\psi}_j=\lambda_j^{(-1)}\YY\VV\SSigma^{(-1)}\bm{w}_j$, 
which is called \textit{DMD modes}.

\section{Koopman Spectral Analysis in vvRKHSs for extracting dynamical structure among observables}
\label{sec:formulation}

Since the existing DMD algorithms basically estimate the spectra of Koopman operators defined in spaces of scalar observables, dependencies among observables are not taken into consideration.
Therefore, they are in principle not applicable to analyze NLDSs with structure among observables.
In this section, we formulate our method by considering the Koopman spectral analysis of such NLDSs in vvRKHSs endowed with kernels for vector-valued functions.

First, let $\mathcal H_{\bm{K}}$ be the vvRKHS endowed with a symmetric positive semi-definite kernel matrix $\bm{K}\colon \mathcal{M} \times \mathcal{M} \rightarrow \R^{m \times m}$~\citep{Alvarez12}. 
That is, $\mathcal{H}_{\bm{K}}$ is a Hilbert space of functions $\bm{f}'\colon\mathcal{M}\to\mathbb{R}^m$, such that for every $\bm{c}\in\mathbb{R}^m$ and $\bm{x}\in\mathcal{M}$, $\bm{K}(\bm{x},\bm{x}')\bm{c}$ as a function of $\bm{x}'$ belongs to $\mathcal{H}_{\bm{K}}$ and, moreover, $\bm{K}$ has the reproducing property
\begin{equation}
\left<\bm{f}',\bm{K}(\cdot,\bm{x})\bm{c}\right>_{\bm{K}} = \bm{f}'(\bm{x})^T\bm{c},
\end{equation}
where $\left<\cdot,\cdot\right>_{\bm{K}}$ is the inner product in $\mathcal{H}_{\bm{K}}$.

In our formulation, we model the relation with a vector-valued observation function $\bm{g}$ for Koopman spectral analysis of NLDSs with structure among observables.
Then, we assume that the components of $\bm{g}$ follow a Gaussian process given by a covariance kernel matrix.
That is, the vector-valued observation function $\bm{g}\colon \mathcal{M} \rightarrow \R^m$ follows the Gaussian distribution
\begin{equation}\label{eq:GP}
\bm{g}(\bm{x}) \thicksim \mathcal{N}(\bm{\mu}(\bm{x}), \bm{K}(\bm{x},\bm{x})), 
\end{equation}
where $\bm{\mu} \in \R^m$ is a vector whose components are the mean functions ${\mu_i (\bm{x})}$ for $\bm{x} \in \mathcal{M}\subset \R^{p}$ and $i = 1,\ldots,m$, and $\bm{K}$ is the above matrix-valued function.
The entries $\bm{K}(\bm{x}, \bm{x})_{i,j}$ in the matrix $\bm{K}(\bm{x}, \bm{x})$ correspond to the covariances between the observables $g_i (\bm{x})$ and $g_{j} (\bm{x})$ for $i,j = 1,\ldots,m$. 
In our following formulation in the vvRKHS determined by $\bm{K}$, again, it is necessary for $\bm{K}(\bm{x}, \bm{x})$ to be a symmetric positive semidefinite matrix.
Practically, for example, we can use positive semidefinite (scalar-valued) kernels between observables as the components of $\bm{K}$.  
 
Based on the above setting, we consider Koopman spectral analysis for NLDSs with structures among observables by extending the formulation of a scalar observation function of DMD with reproducing kernels~\citep{Kawahara16} to that of relations within the vector-valued observable function in vvRKHSs~\citep{Alvarez12}. 
To this end, we first assume that the vector-valued observation function $\bm{g}$ is in the vvRKHS defined by $\bm{K}$, i.e., $\bm{g} \in \mathcal{H}_{\bm{K}}$.
Then, the Koopman operator $\mathcal{K}_{\bm{K}}\colon \mathcal{H}_{\bm{K}}\to \mathcal{H}_{\bm{K}}$ defined by $\mathcal{K}_{\bm{K}}\bm{g}=\bm{g}\circ\bm{f}$, like Eq.~\eqref{eq:koopman}, is a linear operator in $\mathcal{H}_{\bm{K}}$.
Additionally, we denote by $\bm{\phi}_{\bm{c}}\colon \mathcal M\to \mathcal H_{\bm{K}}$ the feature map, i.e., $\bm{\phi}_{\bm{c}}(\bm{x}) = \bm{K}(\cdot,\bm{x})\bm{c}$ for any $\bm{c}\in\mathbb{R}^m$.
According to \cite{Schrodl09}, this is the second type of the feature map in the vvRKHS which directly maps to a Hilbert space $\mathcal{H}_{\bm{K}}$ and has been used in \cite{Evgeniou04,Micchelli05b,Evgeniou05}.
Now, for every $\bm{c}$, we have the following proposition:
\begin{proposition}
\label{PropPF}
Assume $\bm{g} \in \mathcal{H}_{\bm{K}}$. 
Then, $({\cal{K}}_{\bm{K}}{\bm g})(\bm{x})=({\bm g}\circ\bm{f})(\bm{x})$, which is in the case of the vector-valued observable in Eq.~(\ref{eq:koopman}), equals to the following:
\begin{equation}\label{eq:PF}
({\mathcal{K}_{\bm{K}}}^*{\bm{\phi}_{\bm{c}}})(\bm{x})=(\bm{\phi}_{\bm{c}}\circ\bm{f})(\bm{x}) ~~\forall \bm{x} \in \mathcal{M},~ \forall\bm{c}\in\mathbb{R}^m.
\end{equation}
\end{proposition}
\begin{proof}
Since $\mathcal{K}_{\bm{K}} \bm{g} \in \mathcal{H}_{\bm{K}}$, we have, for any $\bm{c}\in\mathbb{R}^m$, $\mathcal{K}_{\bm{K}}\bm{g}(\bm{x})^\top\bm{c} = \left<\mathcal{K}_{\bm{K}}\bm{g},\bm{\phi}_{\bm{c}}(\bm{x})\right>_{\bm{K}} =  \left<\bm{g},\mathcal{K}_{\bm{K}}^*\bm{\phi}_{\bm{c}}(\bm{x})\right>_{\bm{K}}$ for all $\bm{x} \in \mathcal M$.
Similarly, $\bm{g}\circ\bm{f}(\bm{x})^\top\bm{c} =  \left<\bm{g},\bm{\phi}_{\bm{c}}(\bm{f}(\bm{x}))\right>_{\bm{K}}$ because $\bm{f}(\bm{x}) \in \mathcal{M}$.
As a result, since $({\cal{K}}_{\bm{K}}{\bm g})(\bm{x})=({\bm g}\circ\bm{f})(\bm{x})$, we obtain $({\mathcal{K}_{\bm{K}}}^*\bm{\phi}_{\bm{c}})(\bm{x})=(\bm{\phi}_{\bm{c}}\circ\bm{f})(\bm{x})$. 
\end{proof}
The adjoint of the Koopman operator ${\cal{K}}_{\bf{K}}^*$ (also known as the Perron-Frobenius operator) in this case acts as a linear operator in the space spanned by features $\bm{\phi}_{\bm{c}}(\bm{x})$ for $\bm{x} \in \mathcal{M}$. Here, we denote the eigendecomposition of $\mathcal{K}_{\bm{K}}^*$ by $\mathcal{K}_{\bm{K}}^*\bm{\varphi}_j=\lambda_j\bm{\varphi}_j$.


For the practical implementation of the spectral decomposition of the linear operator, we usually need to project data onto directions that are effective in capturing the properties of data, like the standard DMD described in Section~\ref{sec:koopman}.
In DMD with reproducing kernels~\cite{Kawahara16}, a kernel principal orthogonal direction is used for this purpose. 
However, such a projection is not straightforward for the current problem because the principal directions are not defined analogously for tensor data. 
Now, for a given finite time span $[0,\tau]$, we define $\mathcal M_1:=[\bm{\phi}_{\bm{c}}({\bm{x}_0}),..,\bm{\phi}_{\bm{c}}({\bm{x}_{\tau -1}})] $ and $\mathcal M_2:=[\bm{\phi}_{\bm{c}}({\bm{x}_1}),..,\bm{\phi}_{\bm{c}}({\bm{x}_{\tau}})]$.
Then, we adopt the projection onto some orthogonal directions $\bm{\nu}_j = \sum_{t=0}^{\tau-1}\alpha_{j,t}\bm{\phi}_{\bm{c}}(\bm{x}_t) = \mathcal M_1 \bm{\alpha}_{j}$ for $j = 1,\ldots,p$, where the coefficients $\alpha_{j,t} \in \R$ and $\bm{\alpha}_{j} \in \R^{\tau}$ are computed based on a tensor decomposition (described in detail in Section~\ref{sec:tensorDMD}).
Let $\mathcal{U} =[\bm{\nu}_1,\ldots,\bm{\nu}_p]$ and $\mathcal{U} = \mathcal M_1 \bm{\alpha}$ with the coefficient matrix $\bm{\alpha} \in \R^{\tau\times p}$. %
Since ${\mathcal M}_2={\mathcal K}_{\bm{K}}^*{\mathcal M}_1$, the projection of ${\mathcal K}_{\bm{K}}^*$ onto the space spanned by $\mathcal U$ is given as follows:
\begin{flalign}\label{eq:solutionGDMD}
\bm{\hat{F}}=\mathcal U^*\mathcal{K}_{\bm{K}}^*\mathcal U = \bm{\alpha}^*({\mathcal M}_1^*\mathcal M_2)\bm{\alpha}. 
\end{flalign}
Then, if we let $\hat{\bm{F}}=\bm{\hat T}^{-1}\bm{\hat \Lambda}\bm{\hat T}$ be the eigendecomposition of $\hat{\bm{F}}$, we obtain $p$ DMD modes as $\bm{\psi}_{j}=\mathcal{U}{\bm{b}}_{j}$ for $j = 1,\ldots,p$, where $\bm{b}_{j}$ is the $j$th row of $\bm{\hat T}^{-1}$.
The diagonal matrix $\bm{\hat \Lambda }$ comprising the eigenvalues represents the temporal evolution.
To establish the above, we show the following theorem:
\begin{theorem}\label{thm:thm1}
Assume that $\bm{\varphi}_j(\bm{x})^\top\bm{c} = \left<\bm{\kappa}_j,\bm{\phi}_{\bm{c}}(\bm{x})\right>_{\bm{K}}$ for some $\bm{\kappa}_j \in \mathcal{H}_{\bm{K}}$ and $\forall \bm{x} \in \mathcal{M}$.
If $\bm{\kappa}_j$ is in the subspace spanned by $\bm{\nu}_j$, so that $\bm{\kappa}_j = \mathcal{U} \bm{a}_j$ for some $\bm{a}_j \in \C^p$ and ${\mathcal{U}} = [\bm{\nu}_1,\ldots,\bm{\nu}_p]$, then $\bm{a}_j$ is the left eigenvector of $\bm{\hat{F}}$ with eigenvalue $\lambda_j$, and also we have
\begin{equation}\label{eq:Theorem1}
\bm{\phi}_{\bm{c}}(\bm{x}) = \sum^p_{j=1}(\bm{\varphi}_j (\bm{x})^\top\bm{c}) \bm{\psi}_j,
\end{equation}
where $\bm{\psi}_{j}=\mathcal{U}{\bm{b}}_{j}$ and ${\bm{b}}_{j}$ is the right eigenvector of $\hat{\bm{F}}$.
\end{theorem}

\begin{proof}
Since $\mathcal{K}_{\bm{K}}^*\bm{\varphi}_ j$$=$$\lambda_j\bm{\varphi}_j$, we have $\left<\bm{\phi}_{\bm{c}}(\bm{f}(\bm{x})), \bm{\kappa}_j \right>_{\bm{K}}$ $=\lambda_j \left< \bm{\phi}_{\bm{c}}(\bm{x}), \bm{\kappa}_j \right>_{\bm{K}}$. 
Thus, from the assumption, 
\begin{equation}
\left< \bm{\phi}_{\bm{c}}(\bm{f}(\bm{x})), \mathcal{U} \bm{a}_j \right>_{\bm{K}} = \lambda_j \left< \bm{\phi}_{\bm{c}}(\bm{x}), \mathcal{U} \bm{a}_j \right>_{\bm{K}}.
\end{equation}
By evaluating at $\bm{x}_0, \bm{x}_1, \ldots, \bm{x}_{\tau-1}$ and then stacking, we have 
$(\mathcal{U} \bm{a}_j)^*\mathcal{M}_2 = \lambda_j(\mathcal{U} \bm{a}_j)^*\mathcal{M}_1$.
If we multiply $\bm{\alpha}$ from the right-hand side, this gives 
\begin{equation}
\bm{a}_j^*\bm{\alpha}^*\mathcal{M}_1^*\mathcal{M}_2\bm{\alpha} =\lambda_j\bm{a}_j^*.
\end{equation}
Since $\bm{\alpha}^*{\mathcal M}_1^*\mathcal M_2\bm{\alpha} = \bm{\hat{F}}$, 
this means $\bm{a}_j$ is the left eigenvector of $\hat{\bm{F}}$ with eigenvalue $\lambda_j$. 
Let $\bm{b}_j$ be the right eigenvector of $\hat{\bm{F}}$ with eigenvalue $\lambda_j$ and the corresponding left eigenvector $\bm{a}_j$.
Assuming these have been normalized so that $\bm{a}_i^*\bm{b}_j = \delta_{ij}$, then any vector $\bm{h} \in \C^p$ can be written as $\bm{h} = \sum_{j=1}^p(\bm{a}_j^*\bm{h})\bm{b}_j$.
Applying this to $\mathcal{U}^*\bm{\phi}_{\bm{c}}(\bm{x})$ gives 
\begin{equation}
\mathcal{U}^*\bm{\phi}_{\bm{c}}(\bm{x}) = \sum_{j=1}^p(\bm{a}_j^*\mathcal{U}^*\bm{\phi}_{\bm{c}}(\bm{x}))\bm{b}_j 
=\sum_{j=1}^p(\bm{\varphi}_j(\bm{x})^\top\bm{c})\bm{b}_j.
\end{equation}
Since $\bm{b}_j = \mathcal{U}^*\bm{\psi}_j$, this proves Eq. \eqref{eq:Theorem1}.
\end{proof}

The assumptions in the theorem mean that the data are sufficiently rich and thus a set of the orthogonal direction $\mathcal{U}$ gives a good approximation of the representation with the eigenfunctions of $\mathcal{K}_{\bm{K}}^*$. 
As in the case of Eq.~\eqref{eq:decomposition}, by the iterative applications of $\mathcal{K}_{\bm{K}}^*$, we obtain
\begin{equation}\label{eq:GDMDmode}
\vspace*{-0mm}
\bm{\phi}_{\bm{c}}(\bm{x}_t) = \sum^p_{j=1}\lambda_j^t(\bm{\varphi}_j(\bm{x}_0)^\top\bm{c}) \bm{\psi}_j.
\vspace*{-0mm}
\end{equation}
Thus, this theorem gives the connection between the above eigen-values\,/\,-vectors and the Koopman eigen-values\,/\,-functions.

In summary, the formulation first needs the sequence of the kernel matrices $\bm{K}({\bm{x}_t},\bm{x}_t)$ for $t=0,\ldots,\tau$ and then obtains the Koopman spectra of NLDSs with structures among observables by the decomposition of the feature map $\bm{\phi}_{\bm{c}} = \bm{K}(\cdot,\bm{x})\bm{c}$, described in Eq. (\ref{eq:GDMDmode}).
From the above claims in the vvRKHS, it is seemingly necessary to give some ${\bm c}$ for $\bm{\phi}_{\bm c}(\bm{x}_t)$ for its implementation. 
However, we do not require to give ${\bm c}$ because we do not need to directly compute $\bm{\phi}_{\bm{c}} = \bm{K}(\cdot,\bm{x})\bm{c}$ but just need to compute the realization of $\mathcal{U}$ from the observed data.
Concretely, for an implementation of the above analysis, we first regard the given or calculated matrices as a realization of the structure of
the kernel matrices $\bm{K}({\bm{x}_t},\bm{x}_t)$ (see Section \ref{sec:tensorDMD}). 
We denote the realized matrices as $\bm{A}_t \in \R^{m\times m}$ for $t=0,\ldots,\tau$ or the tensor as ${\cal{A}}\in \mathbb{R}^{m\times m\times (\tau+1)}$.
Second, we need to compute a projected matrix in the space spanned by the columns of~$\mathcal U$ (see~\ref{ssec:PIT} and Appendix A for the relation), and then DMD solution $\hat{\FF} \in \R^{p \times p}$ and DMD modes $\bm{\psi}_j \in \C^{m \times m}$ for $j = 1,\ldots,p$ (see~\ref{ssec:rtdmd} and \ref{ssec:dmdao}). 
In the next section, we develop the procedure by reformulating TDMD for computing these quantities from data.

\section{Reformulated Tensor-based DMD}
\label{sec:tensorDMD}

Here, we reformulate TDMD by \cite{Klus18} as an estimation algorithm for the above formulation using the sequence of kernel matrices $\bm{K}({\bm{x}_t},\bm{x}_t)$ for $t=0,\ldots,\tau$, i.e., to calculate the above quantities without breaking the dependent structure among observables.
We first review the tensor-train (TT) format in Subsection~\ref{ssec:tt}, and then compute the projected matrix and the compute pseudo-inverse of a tensor in Subsection~\ref{ssec:PIT}. Next, we reformulate TDMD (we call it {\em reformulated TDMD}) in Subsection~\ref{ssec:rtdmd} and finally describe DMD for our problem, i.e., NLDSs with structures among observables in Subsection~\ref{ssec:dmdao}.
Note that although TDMD is applicable for analyzing higher-order complex dynamical systems, our problem considers a sequence of the matrices ${\cal{A}}\in \mathbb{R}^{m\times m\times (\tau+1)}$ as an input tensor, which is a sequence of the realization of the $\bm{K}({\bm{x}_t},\bm{x}_t)$'s structure.

\subsection{TT-format}
\label{ssec:tt}
In general, it is known that analyzing high-dimensional data becomes infeasible due to the so-called curse of dimensionality.
This could be moderated by exploiting low-rank tensor approximation approaches. 
Several tensor formats such as the canonical format, Tucker format, and TT-format have been developed for this purpose (see e.g.,~\cite{Grasedyck13}). Among these formats, the TT-format is known to be relatively stable and scalable for high-order tensors compared with the other formats~\citep{Oseledets11}. 

Here, we review the TT-format.
Let $ {\cal{A}} \in \mathbb{C}^{n_1 \times \dots \times n_d} $ be an order-$d$ tensor, where $n_l$ denotes the dimensionality of the $l$-th mode for $l = 1,\ldots,d$ (called \emph{full-format}).
In TT-decomposition (see \cite{Oseledets11}), ${\cal{A}}$ is decomposed into $d$ core tensors $ {\cal{A}} ^{(l)} \in  \mathbb{C}^{r_{l-1} \times n_l \times r_l}$, where $r_0 = r_d = 1 $.
$r_l$ is called TT-rank, which controls the complexity of TT decomposition.
For an elementary expression, any element of ${\cal{A}}$ is given by
${\cal{A}}_{i_1, \dots, i_d}
= \sum_{k_0=1}^{r_0} \dots \sum_{k_{d}=1}^{r_{d}}
{\cal{A}}^{(1)}_{k_0, i_1, k_1} \cdot \dotsc \cdot {\cal{A}}^{(d)}_{k_{d-1}, i_d, k_d},
$ where the subscripts of the tensors denote the indices.
Moreover, for two vectors $\bm{v} \in \C^{n_1}$ and $\bm{w} \in \C^{n_2}$, the tensor product $\bm{v} \otimes \bm{w} \in \C^{n_1 \times n_2} $ is given by $(\bm{v} \otimes \bm{w})_{i,j} = (\bm{v} \cdot  \bm{w}^\top)_{i,j} = v_i \cdot w_j$.
Using the tensor product, the whole tensor can then be represented as
${\cal{A}} = \sum_{k_0=1}^{r_0} \dots \sum_{k_d=1}^{r_d} {\cal{A}}^{(1)}_{k_0,:,k_1} \otimes \dots \otimes {\cal{A}}^{(d)}_{k_{d-1},:,k_d},
$
where colons are used to indicate all components of the mode, e.g., $ {\cal{A}}^{(l)}_{k_{i-1},:,k_i} \in \C^{n_l} $. 

To describe the matricizations and vectorizations (also called \emph{tensor unfoldings}) for efficient computation, let ${\cal{A}}_{i_1, \dots,i_{l},:,i_{l+1} \dots,i_d}$ denote an $n_l$-dimensional vector called the mode-$l$ fiber, where $ 1 \leq l \leq d-1 $.
For the two ordered subsets $N' = \{n_{1}, \ldots, n_{l}\}$ and $N'' = \{n_{l+1}, \ldots, n_{d}\}$ of $ N = \{n_1, \ldots, n_{d}\} $, the matricization of ${\cal{A}}$ with respect to $ N' $ and $ N''$ is denoted by
$\mat{{\cal{A}}}{N'}{N''} \in \C^{(n_1 \cdot \dotsc \cdot n_l) \times (n_{l+1} \cdot \dotsc \cdot n_d)} ,
$
which is defined by concatenating the mode-$l$ fibers of $\cal{A}$.
In the special case with $N' = N$ and $N'' = { \varnothing } $, the vectorization of $\cal{A}$ is given by $\vectorize({\cal{A}}) \in \C^{n_1 \cdot \dotsc \cdot n_d}$.


\subsection{Projected Matrix and Modified Pseudo-inverse for TT-format}
\label{ssec:PIT}

In this section, before reformulating the TDMD in TT-format, we modify the computation of the pseudo-inverse of a tensor in \citep{Klus18} and then obtain the projected matrix in the space spanned by the columns of $\mathcal U$ in Section~\ref{sec:formulation}. 
Note that although our problem considers a sequence of matrices ${\cal{A}}\in \mathbb{R}^{m\times m\times (\tau+1)}$ as an input tensor, TDMD is applicable for analyzing higher-order complex dynamical systems.  
For TDMD, consider $ \tau $ snapshots of $ d $-dimensional tensor trains ${\cal{X}}, {\cal{Y}} \in \C^{n_1 \times \dots \times n_d \times \tau} $, where $ {\cal{X}}_{:,\ldots,:,i+1}  \in \C^{n_1 \times \dots \times n_d}$ for $i = 0,\ldots,\tau-1$ and $ {\cal{Y}}_{:,\ldots,:,i}$ for $i = 1,\ldots,\tau$.
Let $r_0,  \dotsc, r_{d+1}$ and $s_0,  \dotsc, s_{d+1}$ be the TT-ranks of ${\cal{X}}$ and ${\cal{Y}}$, respectively. Now, let $\bm X,\bm Y \in \C^{n_1 \cdot \dotsc \cdot n_d \times \tau}$ be the specific matricizations of ${\cal{X}}$ and ${\cal{Y}}$, where we contract the dimensions $n_1, \dotsc, n_d$ such that every column of $\bm X$ and $\bm Y$ is the vectorization of the corresponding $\tau = n_{d+1}$ snapshot, respectively. 
 
To efficiently compute TDMD only with matrix products (without any tensor products), we first perform TT-decomposition of $\mathcal X$ and matricize to $\bm{X}$ as
\begin{equation} \label{eq:TTdecomposition}
\vspace*{-0mm}
    \bm{X} = \bm{M} \bm{\Sigma} \bm{N},
\vspace*{-0mm}
\end{equation}
where $\bm{M} = \mat{\left( \sum_{k_0 =1}^{r_0} \cdots \sum_{k_{d-1}=1}^{r_{d-1}} {\cal{X}}^{(1)}_{k_0, :, k_1 } \otimes \dotsc \otimes {\cal{X}}^{(d)}_{k_{d-1}, :, : } \right)}{n_1, \dotsc, n_d}{r_{d}}$, $\bm{N} = \mat{\left({\cal{X}}^{(d+1)}_{:, :, k_{d+1} } \right)}{r_d}{\tau}$, and $\bm{\Sigma}$ is a diagonal matrix with singular values in its diagonal elements computed by the SVD of $\mat{{\cal{X}}^{(d)}}{r_{d-1}, n_{d}}{r_{d}}$. 
$\bm{N}$ is equivalent to the last core ${\cal{X}} ^{(d+1)}$, which is a matrix because $r_{d+1} = k_{d+1} = 1$. 
Note that this is similar to SVD in the matrix form, but SVD and this matricization after TT-decomposition are completely different.
$\bm{M} \in \C^{n_1\cdots n_d \times r_d}$ computed by the first $d$ core of $\cal{X}$ is left-orthogonal
\footnote{In general, a matrix $\bm{A}$ is left-orthonormal if 
$\bm{A}^*  \bm{A} = \bm{I}$ and right-orthonormal if 
$\bm{A} \bm{A}^*  = \bm{I} $.} 
due to the procedure of TT-decomposition algorithm~\citep{Oseledets11}, and reflects some part of tensor structure of $\mathcal A$ when folding in full-format.  
In our problem for NLDSs with dependent structures among observables, 
$\bm{M} = \bm{X}\bm{N}^\dagger \bm{\Sigma}^{-1} \in \R^{m^2\times r_d}$ for $\bm{X} \in \R^{m^2 \times \tau}$ works as the projected matrix in the space spanned by the columns of $\mathcal U = {\mathcal M}_1\bm{\alpha}$ in Section~\ref{sec:formulation}.
For the details of the relation, see Appendix A. 

Next, we claim that the pseudo-inverse $\bm{X}^\dagger$ for the computations of the following TDMD is computed as shown in the following proposition:\footnote{The computation of the pseudo-inverse $\bm{X}^\dagger$ in \citep{Klus18} is described with the left and right-orthonormalization of the cores of $\mathcal X$ including QR decompositions in a mathematically general way. However, when considering TDMD (i.e., $l = d-1$ and $\bm N$ being a matrix), the proposed algorithm with the pseudo-inverse of $\bm N$ using SVD (without QR decomposition) is more directly and stably computable than the previous algorithm. The difference in the computational efficiency between them depends on the problem such as the TT-ranks and dimensions of the tensor.}

\begin{proposition}
\label{PropPI}
Assume that $\bm{X} \in \C^{n_1 \cdot \dotsc \cdot n_d \times \tau}$ matricized from $\mathcal X \in \C^{n_1 \times \dots \times n_d \times \tau}$ is decomposed as in Eq.~\eqref{eq:TTdecomposition}. 
Then, the pseudo-inverse $\bm X^\dagger$ is given by 
\begin{equation} \label{eq:representation of X^+}
    \bm{X}^\dagger = \bm{N}^\dagger \, \bm{\Sigma}^{-1} \, \bm{M}^*.
\vspace*{-0mm}
\end{equation}
\end{proposition} 

\begin{proof}
Although $\bm M$ is left-orthogonal as mentioned above,
the last core $\bm{N} = {\cal{X}}^{(d+1)} \in \C^{r_d\times n_{d+1}}$ is not right-orthogonal, i.e., $\bm{N}\cdot \bm{N}^* \neq \bm{I}$. 
Then, we can use the pseudo-inverse matrix $\bm{N}^\dagger \in \C^{n_{d+1}\times r_d}$, i.e., $\bm{X}^\dagger = \bm{N}^\dagger \, \bm{\Sigma}^{-1} \, \bm{M}^*$.
Since $\bm{M}^*\cdot \bm{M} = \bm{I}$ and $\bm{N} \cdot \bm N^\dagger = \bm I$, it follows that the pseudo-inverse $\bm{X}^\dagger$ satisfies the necessary and sufficient conditions for the pseudo-inverse, i.e., it satisfies the following four equations:
\newpage 
\begin{align} 
\bm{X} \bm{X}^\dagger \bm{X} = \bm{M} \, \bm{\Sigma} \, \bm{N} \cdot \bm{N}^\dagger \, \bm{\Sigma}^{-1} \, \bm{M}^* \cdot \bm{M} \, \bm{\Sigma} \, \bm{N} = \bm{X},\\
\bm{X}^\dagger \bm{X} \bm{X}^\dagger = \bm{N}^\dagger \, \bm{\Sigma}^{-1} \, \bm{M}^* \cdot \bm{M} \, \bm{\Sigma} \, \bm{N} \cdot \bm{N}^\dagger \, \bm{\Sigma}^{-1} \, \bm{M}^* = \bm{X}^\dagger,\\
(\bm{X} \bm{X}^\dagger)^* = (\bm{M} \, \bm{M}^*)^* = \bm{M} \, \bm{M}^* = \bm{X} \bm{X}^\dagger,\\
(\bm{X}^\dagger \bm{X})^* = (\bm{N}^\dagger \, \bm{N})^* = \bm{N}^\dagger \, \bm{N} = \bm{X}^\dagger \bm{X}.
\end{align} 
For the fourth equation, we use the property of pseudo-inverse $(\bm{N}^\dagger \, \bm{N})^* = \bm{N}^\dagger \, \bm{N}$. 
\end{proof}
\subsection{Reformulated TDMD}
\label{ssec:rtdmd}
\subsubsection{TDMD solution}
Using similar matricizations of $\mathcal X$, we can also represent the tensor unfolding $\bm{Y}$ as a matrix product, i.e., 
\begin{equation} \label{eq:representation of Y}
        \bm{Y} = \mat{\left(\sum_{l_{0} =1}^{s_{0}} \cdots \sum_{l_{d-1}=1}^{s_{d-1}} {\cal{Y}}^{(1)}_{l_0, :, l_{1} } \otimes \dotsc \otimes {\cal{Y}}^{(d)}_{l_{d-1}, :, : }\right) }{n_1, \dotsc, n_d}{s_{d+1}} 
         \cdot \mat{{\cal{Y}}^{(d+1)}}{s_{d+1}}{\tau}
           = \bm{P} \, \bm{Q}.
\end{equation}
We abbreviate the indices of ${\cal{Y}}^{(d+1)}_{:, :, k_{d+1} }$ as ${\cal{Y}}^{(d+1)}$ because $r_{d+1} = k_{d+1} = 1$.
Note that we do not require any special property of the tensor cores of ${\cal{Y}}$.
Combining the representations of $\bm{X}^\dagger$ and $\bm Y$ and generalizing the basic DMD procedure in Section~\ref{sec:koopman} to the tensor form, we can express the rank-reduced DMD solution $\hat{\bm F} \in \C^{r_{d+1} \times r_{d+1}}$ (equivalent to $\hat{\bm F}$ of Graph DMD in Section~\ref{sec:formulation}) as 
\begin{equation} \label{eq:reduced matrix in TT}
    \hat{\bm F} = \bm{M}^* \bm{Y} \cdot \bm{X}^\dagger \bm{M} =  \bm{M}^* \cdot \bm{P} \, \bm{Q} \cdot \bm{N}^\dagger  \, \bm{\Sigma}^{-1}.
\end{equation}
To compute $\hat{\bm F}$ in Eq. (\ref{eq:reduced matrix in TT})
, we bypass this computational cost by splitting Eq. (\ref{eq:reduced matrix in TT}) 
into different parts. 
First, we consider that in the rank-reduced $\bm{M}^* \cdot \bm{P} \in \R^{r_d \times s_d}$, any entry is given by

\begin{equation}
   ( \bm{M}^* \cdot \bm{P} )_{i,j} = 
    \sum_{{k_0 =1}}^{r_0} \cdots \sum_{k_{d-1}=1}^{r_{d-1}} \sum_{l_{0}=1}^{s_{0}} \cdots \sum_{{l_{d-1}=1}}^{s_{d-1}} 
\left( {\cal{X}}^{(1)}_{k_0, :, k_1 } \right)^T {\cal{Y}}^ {(1)}_{l_0, :, l_1} \cdot \dotsc \cdot \left( {\cal{X}}^{(d)}_{k_{d-1}, :, i } \right)^T {\cal{Y}}^ {(d)}_{l_{d-1}, :, j}.
\end{equation}
This is based on the following computation: 
$\vectorize({\cal{X}})^T \cdot \vectorize({\cal{Y}}) = \Pi_{l=1}^d \left( {\cal{X}}^{(l)}\right)^T \cdot {\cal{Y}}^{(l)}$.
In this way, we can compute $\bm{M}^* \cdot \bm{P}$ without leaving the TT-format, and we only have to reshape certain contractions of the TT-cores. This computation can be implemented efficiently using Algorithm 4 from \citep{Oseledets11}. The result assumes that the TT-ranks of ${\cal{X}}$ and ${\cal{Y}}$ are small compared to the entire state space of these tensors. Indeed, the tensor ranks $r_d$ and $s_d$ are both bounded by the number of snapshots $\tau$. 
Second, for $\bm{Q} \cdot \bm{N}^\dagger$ in Eq. (\ref{eq:reduced matrix in TT})
, we simply obtain
\begin{equation}
    \bm{Q} \cdot \bm{N}^\dagger = \mat{{\cal{Y}}^{(d+1)}}{s_{d}}{\tau} \cdot \left( \mat{{\cal{X}}^{(d+1)}}{r_{d}}{\tau} \right)^\dagger.
\end{equation}
In this computation, we do not need to convert any tensor products of the cores of ${\cal{X}}$ or ${\cal{Y}}$, into full tensors during our calculations.

\subsubsection{TDMD mode} 
Next, we consider the computation of the DMD modes of $\hat{\bm F}$. If $ \lambda_1, \dotsc, \lambda_{p} $ are the eigenvalues of $ \hat{\bm F} $ corresponding to the eigenvectors $ \bm{w}_1, \dotsc, \bm{w}_{p} \in \C^{r_{d+1}} $, then the vectorized DMD modes $ \bm{\varphi}_1, \dotsc, \bm{\varphi}_{p} \in \C^{n_1 \cdot \ldots \cdot n_d} $ of $\bm{F}$ (as in Section~\ref{sec:koopman}) are given by $\bm{\varphi}_j = (1/\lambda_j)\cdot \bm{P} \, \bm{Q} \cdot \bm{N}^\dagger \, \bm{\Sigma}^{-1} \cdot \bm{w}_j,$ for $j = 1, \dotsc, p$. Tensor representation ${\cal{Z}} \in \C^{n_1 \times \ldots \times n_d \times p}$ including all DMD modes is given by
\begin{equation}
    {\cal{Z}} = \sum_{l_0 =1}^{s_0} \cdots \sum_{l_{d}=1}^{s_{d}} {\cal{Y}}^{(1)}_{l_0, :, l_1 } \otimes \dotsc \otimes {\cal{Y}}^{(d)}_{l_{d-1}, :, l_d } 
    \otimes {\left(\bm{Q} \cdot \bm{N}^\dagger \, \bm{\Sigma}^{-1} \cdot \bm{W} \cdot \bm{\Lambda}^{-1} \right)} \,^{\phantom{d}}_{l_d, :},
\end{equation}
again with $\vectorize({\cal{Z}}_{:, \dotsc, :, j}) = \bm{\varphi}_j$ and $\bm{\Lambda}$ is a diagonal matrix arranging $\lambda_1, \ldots, \lambda_{p}$. 

The overall algorithm of the reformulated TDMD is shown in Algorithm \ref{alg:TDMD}.
We can express the DMD modes using given tensor trains $ {\cal{X}} $ and ${\cal{Y}}$, modifying just the last core.
In this case, we benefit from not leaving the TT-representations of $ {\cal{X}} $ and ${\cal{Y}}$. 
In other words, the bottleneck of this algorithm regarding scalability would be sequential SVDs in TT-decompositions of $ {\cal{X}} $ and ${\cal{Y}}$.

\begin{algorithm}[h!]
   \caption{Reformulated Tensor-based DMD}
   \label{alg:TDMD}
\begin{algorithmic}[1]
   \STATE {\bfseries Input:} $\mathcal{X}, \mathcal{Y} \in \C^{n_1 \cdot \dotsc \cdot n_d \times \tau}$
   \STATE {\bfseries Output:} dynamic mode tensor $\mathcal{Z}$ and eigenvalue matrix $\bm \Lambda$
   \STATE $\bm M$, $\bm \Sigma$, $\bm N$ $\leftarrow$ matricized after decomposition of $\mathcal X$;
   \STATE $\bm N^\dagger$ $\leftarrow$ pseudo-inverse of $\bm N$;
   \STATE $\bm P$, $\bm Q$  $\leftarrow$ matricized after decomposition of $\mathcal Y$;
   \STATE $\hat{\bm F}$ $\leftarrow$  $(\bm M^* \cdot \bm P)(\bm Q \cdot \bm N^\dagger)\bm{\Sigma}^{-1}$;
   \STATE $\bm \Lambda, \bm W$ $\leftarrow$ eigendecomposition of $\hat{\bm F}$;
   \STATE $\mathcal{Z}$ $\leftarrow$ $\sum_{l_0 =1}^{s_0} \cdots \sum_{l_{d}=1}^{s_{d}} {\cal{Y}}^{(1)}_{l_0, :, l_1 } \otimes \dotsc \otimes {\cal{Y}}^{(d)}_{l_{d-1}, :, l_d } 
     \otimes {\left(\bm Q \cdot \bm N^\dagger \, \bm{\Sigma}^{-1} \cdot \bm W \cdot \bm{\Lambda}^{-1} \right)} \,^{\phantom{d}}_{l_d, :}$;
   \STATE {\bfseries return:} $\mathcal{Z}$, $\bm \Lambda$;
\end{algorithmic}
\end{algorithm}

\subsection{DMD for NLDSs with structures among observables}
\label{ssec:dmdao}

In DMD for NLDSs with structures among observables, as a special case of the above reformulated TDMD in \ref{ssec:rtdmd}, we use a sequence of the matrices $\mathcal A \in \R^{m \times m \times (\tau+1)}$ as a sequence of the realization of the $\bm{K}({\bm{x}_t},\bm{x}_t)$'s structure for $t=0,\ldots,\tau$ in reformulated TDMD, i.e., $d=2$ and $n_1 = n_2 = m$. 
Input tensors $\mathcal{X}$ and $\mathcal{Y}$ are created from $\mathcal{A}_{:,:,t}$ and $\mathcal{A}_{:,:,t+1}$ for $t = 0,\ldots,\tau-1$, respectively.
As a result, we obtain DMD modes $\bm{\psi}_j \in \C^{m \times m}$ as in Section~\ref{sec:formulation} by matricizing $\bm{\varphi}_j$ (or $\mathcal{Z}_{:,:,j}$) with eigenvalues $\lambda_j$. 
The overall algorithm is shown in Algorithm \ref{alg:GDMD}.

\begin{algorithm}[h!]
   \caption{DMD for NLDSs with structures among observables}
   \label{alg:GDMD}
\begin{algorithmic}[1]
   \STATE {\bfseries Input:} sequence of the matrices $\mathcal A \in \R^{m \times m \times (\tau+1)}$
   \STATE {\bfseries Output:} dynamic mode matrix $\bm{\psi}_j \in \R^{m \times m}$ and eigenvalue $\lambda_j$
   \STATE $\mathcal{X}, \mathcal{Y} \in \R^{m \times m \times \tau}$ $\leftarrow$ make tensors from $\mathcal{A}$;
   \STATE $\bm M$, $\bm \Sigma$, $\bm N$ $\leftarrow$ matricized after decomposition of $\mathcal X$;
   \STATE $\bm N^\dagger$ $\leftarrow$ pseudo-inverse of $\bm N$;
   \STATE $\bm P$, $\bm Q$  $\leftarrow$ matricized after decomposition of $\mathcal Y$;
   \STATE $\hat{\bm F}$ $\leftarrow$  $(\bm M^* \cdot \bm P)(\bm Q \cdot \bm N^\dagger)\bm{\Sigma}^{-1}$;
   \STATE $\bm \Lambda, \bm W$ $\leftarrow$ eigendecomposition of $\hat{\bm F}$;
   \STATE $\mathcal{Z}$ $\leftarrow$ $\sum_{l_0 =1}^{s_0} \sum_{l_{1}=1}^{s_{1}} \sum_{l_{2}=1}^{s_{2}} {\cal{Y}}^{(1)}_{l_0, :, l_1 } \otimes {\cal{Y}}^{(2)}_{l_{1}, :, l_2} 
     \otimes {\left(\bm Q \cdot \bm N^\dagger \, \bm{\Sigma}^{-1} \cdot \bm W \cdot \bm{\Lambda}^{-1} \right)} \,^{\phantom{2}}_{l_2, :}$;
   \STATE {\bfseries return:} $\bm{\psi}_j = {\cal{Z}}_{:, :, j}$, $\lambda_j = \bm \Lambda_{j,j}$;
\end{algorithmic}
\end{algorithm}

\section{Graph DMD}
\label{sec:gdmd}
In this section, as a special case of DMD for NLDSs with structures among observables in \ref{ssec:dmdao}, we propose the method named as \textit{Graph DMD}, which is a numerical algorithm for Koopman spectral analysis of GDSs.
According to the notation of \citep{Cliff16}, we consider an autonomous discrete-time weighted and undirected GDS defined as
\begin{flalign}\label{eq:GDS}
G = (\mathcal{V},\mathcal{E},\bm{x}_t,\bm{y}_t,\bm{f}, \bm{g}, \bm{A}_t),
\end{flalign}
where $\mathcal{V}=\{V^1,\ldots,V^m\}$ and $\mathcal{E} = \{E^1,\ldots,E^{l}\}$ are the vertex and edge sets of a graph, respectively, fixed at each time $t\in\mathbb{T}$.
$\bm{x}_{t} \in \mathcal{M} \subset \R^{p}$ for the GDS and $\bm{f}\colon\cal{M}\to\cal{M}$ is a (typically, nonlinear) state-transition function (i.e., $\bm{x}_{t+1} = \bm{f}(\bm{x}_t)$).
$\bm{y}_{t} \in \R^{m}$ are observed values that correspond to vertices and are given by $\bm{y}_{t}:= \bm{g}(\bm{x}_t)$, where $\bm{g}\colon \mathcal M\to \R^m$ is a vector-valued observation function. 
$\bm{A}_t \in \R^{m \times m}$ is an adjacency matrix, whose component $a_{i,j,t}$ represents the weight on the edge between $V^i$ and $V^j$ at each time $t$.
For example, the weight represents some traffic volume between the locations in networks or public transportations. Another example of the weights for undirected GDSs is the relation between moving agents (such as distances) in multi-agent systems \cite{Couzin02,Fujii16}.

In Graph DMD, we consider a sequence of adjacency matrices $\bm{A}_t \in \mathbb{R}^{m\times m} $ for $t = 0,...,\tau$ or $\mathcal A \in \R^{m \times m \times (\tau+1)}$ as input.  
Here, we assume that the adjacency matrix $\bm{A}_t$ observed at each time is a realization of the structure of the kernel matrix $\bm{K}({\bm{x}_t},\bm{x}_t)$ in Section \ref{sec:formulation}. 
That is, the weight of $\bm{A}_t$ is assumed to represent the correlation between the observables.
Again, in our formulation in the vvRKHS determined by $\bm{K}({\bm{x}_t},\bm{x}_t)$, it is necessary for $\bm{K}({\bm{x}_t},\bm{x}_t)$ to be a symmetric positive semidefinite matrix (i.e., we consider an undirected graph). 
For an implementation of Graph DMD, we use a sequence of adjacency matrices $\mathcal A \in \R^{m \times m \times (\tau+1)}$ in DMD for NLDSs with structures among observables.
That is, we only replace the sequence of the matrices in Algorithm \ref{alg:GDMD} with a sequence of adjacency matrices.
As a result, similarly in Algorithm \ref{alg:GDMD}, we obtain DMD modes $\bm{\psi}_j \in \C^{m \times m}$ with eigenvalues $\lambda_j$.

\vspace*{-0mm}
\section{Relation to previous works}
\label{sec:related}
\vspace*{-0mm}
\subsection{Dynamic mode decomposition}
Spectral analysis (or decomposition) for analyzing dynamical systems is a popular approach aimed at extracting low-dimensional dynamics from data.
DMD, originally proposed in fluid physics~\citep{Rowley09,Schmid10}, has recently attracted attention also in other areas of science and engineering, including analysis of power systems~\citep{Susuki14}, epidemiology~\citep{Proctor15}, neuroscience~\citep{Brunton16a}, image processing~\citep{Kutz16b,Takeishi17d}, controlled systems~\citep{Proctor16}, and human behaviors~\citep{Fujii17,Fujii18a,Fujii19a}. Moreover, there are several algorithmic variants to overcome the problem of the original DMD such as the use of nonlinear basis functions~\citep{Williams15}, a formulation in a reproducing kernel Hilbert space~\citep{Kawahara16}, in a supervised learning framework via multitask learning~\citep{Fujii19b}, in a Bayesian framework~\citep{Takeishi17}, and using a neural network~\citep{Takeishi17b}. 
For interconnected systems, e.g., Susuki and Mezi{\'c} \citep{Susuki11} computed Koopman modes of coupled swing dynamics in power systems, and Heersink et al. \citep{Heersink17} proposed DMD for (simulated) interconnected control systems, which extends DMD with control~\citep{Proctor16}.
Note that these are basically formulated without considering the structures among observables unlike our formulation described in this paper.

\subsection{Vector-valued RKHSs}
RKHSs of vector-valued function (vvRKHS), endowed with a matrix-valued or operator-valued kernel~\cite{Caponnetto08}, have attracted an increasing interest as the methods to deal with such as classification or regression problem with multiple outputs (e.g., \cite{Alvarez12,Micchelli05}). 
In real-world problems, this approach has applied to such as image processing~\cite{Quang10} and medical treatment effects~\cite{Alaa17}.
Gaussian processes for vector-valued functions have also been formulated using the covariance kernel matrix \cite{Alvarez12}.
We first formulated the modal decomposition methods in the vvRKHS with the assumption that the vector-valued observable follows Gaussian process.
Other researchers performed spatiotemporal pattern extraction by spectral analysis of vector-valued observables using operator-valued kernel \cite{Giannakis18}, but did not formulate in vvRKHSs and directly extract the dynamical information about the dependent structure among observables.

\subsection{Other algorithms for graph data}
For signal processing of a graph, researchers have basically examined the graph property in the graph (spatial) domain such as using graph Laplacians (e.g.,~\cite{Chung97}), graph Fourier transforms (e.g., \cite{Taubin95}), and graph convolutional networks (e.g., \cite{Defferrard16}).
Graph Laplacians and other extensions have been also used for regularization by utilizing data structures (e.g.,\cite{Liu16} ).  
For the graph sequence data in several scientific fields (see Section~\ref{sec:introduction}), various analyses have been examined such as using topological variables~\citep{Bullmore09} and objective variables in simulation~\citep{Centola07,Kuhlman11} in each snapshot of the graph time series, or performed graph abnormality detection~\citep{Ide04} with the temporal sliding windows of the time series.
A few methods have been directly applied to the graph time series such as using graph convolutional networks~(e.g., \citep{Seo16}). 
Meanwhile, our method has advantages to directly extract the underlying low-dimensional dynamics of GDSs.

\vspace*{-0mm}
\section{Experimental Results}
\label{sec:example}
\vspace*{-0mm}

We conducted experiments to investigate the empirical performance of our method (for clarity, we called it Graph DMD in this section) using synthetic data in Subsection~\ref{ssec:synthetic}. 
Then, we examined the applications to extract and visualize specific spatiotemporal dynamics in a real-world bike-sharing system data in Subsection~\ref{ssec:bike} and in fish-schooling simulation data as an example of unknown global dynamics in Subsection~\ref{ssec:fish}.
Note that, as mentioned in Section \ref{sec:related}, most of the conventional methods for a graph have basically extracted the graph property in the graph (spatial) domain. 
Meanwhile, our method directly extracts the underlying low-dimensional dynamics among observables, which cannot be estimated by these methods.
Therefore, we did not compare conventional methods for a graph but compare the conventional DMD algorithms with our method.
 
\subsection{Synthetic Data}
\label{ssec:synthetic}
We first validated the performance of our method to extract the dynamical information on synthetic data.
We generated a sequence of noisy adjacency matrix series $\bm{A}_t \in \R^{D \times D}$ using the following equations: 
\begin{equation}
\bm{A}_t= 0.99^t \bm{A}_{m_1} + 0.9^t \bm{A}_{m_2} + \ee_t,
\end{equation}
where $\bm{A}_{m_1},\bm{A}_{m_2} \in \R^{D \times D}$ (and $\bm{A}_t$ for every $t$) are the adjacency matrices shown in Figure~\ref{fig:Synthetic}a and d, respectively.
Black and white indicate lower and higher values, respectively.
$D$ was set to 64 and 256 for examining the effect of data dimension (Figure 1 is for $D=64$).
Each element of $\ee_t \in \R^{D \times D}$ is independently and identically sampled from a zero-mean Gaussian with variance $1\mathrm{e}{-02}$.
In this case, the true spatial dynamic modes are $\bm{A}_{m_1}$ and $\bm{A}_{m_2}$, with the corresponding DMD eigenvalues 0.99 and 0.9 (mode 1 and mode 2), respectively.
We here estimated the spatial and temporal modes from noisy data using our method (Graph DMD) and the exact DMD \cite{Tu14} as a baseline (described in Section~\ref{sec:koopman}).
In Graph DMD (or reformulated Tensor-based DMD), TT-decomposition parameter $\varepsilon$ (i.e., the tolerance in the successive SVD) is critical for estimating a few DMD mode such as in this case.
The estimation performances are computed using the average values of 10 tasks.

The estimation results are shown in Figure~\ref{fig:Synthetic} and Table \ref{tab:Synthetic}.
The effect of the estimation errors for the two leading eigenvalue was evaluated by the relative errors defined by $ \Delta |\lambda| = | \lambda - \tilde {\lambda} | / |\lambda|$, where $\lambda$ is the estimated eigenvalue and $\tilde{\lambda}$ is the ground truth of the eigenvalues (0.99 and 0.9 for modes 1 and 2, respectively).  
Our proposed method with $ \varepsilon = 1\mathrm{e}{-01}$ was more accurate than that with exact DMD.
Note that in this experiment, the result of our method was the same as exact DMD when $ \varepsilon \leq 1\mathrm{e}{-02}$ and our method extracted only one mode (i.e. one eigenvalue) when $ \varepsilon > 1\mathrm{e}{-01}$.
With respect to the size effect of the adjacency matrix, the larger the size, the higher is the estimation error because of the larger amount of noise.
For the two leading spatial DMD modes, our method with $\varepsilon = 1\mathrm{e}{-02}$ decreased more noise (especially in Figure~\ref{fig:Synthetic}f) than the exact DMD shown in Figure~\ref{fig:Synthetic}e (the results of our method with $ \varepsilon = 1\mathrm{e}{-02}$ are the same as those for the exact DMD). 
In addition, we confirmed that there were almost no differences in the eigenvalues between Graph DMD and the original TDMD\cite{Klus18} ($< 1\mathrm{e}{-12}$ for all eigenvalues).

\begin{figure}[t]
\centering
\includegraphics[width=0.8\columnwidth]{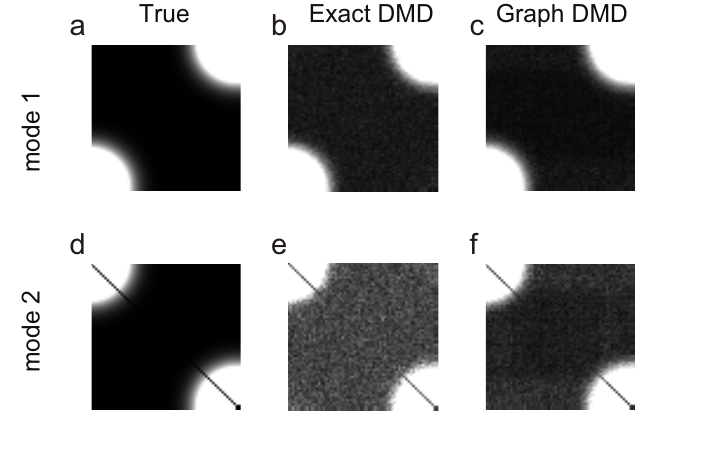}
\vspace{-8mm}
\caption{Two spatial DMD modes for two temporal modes estimated by each method. (a) and (d): ground truth. (b) and (e): results of exact DMD. (c) and (f): results of Graph DMD.
}
\label{fig:Synthetic}
\end{figure}

\begin{table}[h]

    \centering
    \caption{Estimation error of the two leading DMD eigenvalue for the numerical example. The entries show the relative errors $\Delta |\lambda|$ for different values of $ \varepsilon $ and $D$.}
    \label{tab:VertexErrors}
    \newcommand{\md}[2]{\multicolumn{#1}{c|}{#2}}
    \newcommand{\me}[2]{\multicolumn{#1}{c}{#2}}
    \scalebox{0.85}
    {\label{tab:Synthetic}
    \begin{tabular}{|l|rr|rr|}
        \hline
        & \md{2}{size $64 \times 64$} & \md{2}{size $256 \times 256$} \\
        ~~~$\Delta |\lambda|$ & \me{1}{Mode 1} & \md{1}{Mode 2} & \me{1}{Mode 1} & \md{1}{Mode 2} \\ \hline
DMD                                     & 1.49e$-$03 & 9.04e$-$03 & 1.85e$-$02 & 4.05e$-$01 \\
$\varepsilon=1\mathrm{e}{-02} $ & 1.49e$-$03 & 9.04e$-$03 & 1.85e$-$02 & 4.05e$-$01 \\
$\varepsilon=1\mathrm{e}{-01} $ & 9.33e$-$04 & 6.15e$-$03 & 1.85e$-$02 & 3.74e$-$01 \\ \hline
    \end{tabular}}
    \vspace{-3mm}
\end{table}

\subsection{Bike-sharing data}
\label{ssec:bike}
One of the direct applications of our method is to extract the dynamical structure among observables.
In some real-world datasets, we can use prior knowledge about the dynamics such as biological rhythms (e.g., a day, month, and year).
Then, our method can extract the spatial (e.g., graph) coherent structure for the focusing dynamics (e.g., rhythm or frequency).
Here, we extracted and illustrated graph (spatial) DMD modes with real-world bike-sharing system data.
The bike-sharing data consisted of the numbers of bikes returned from one station to another in an hour in Washington D.C\footnote{https://www.capitalbikeshare.com/}.  
We collected a sum of the numbers of bikes transported between the two different stations for both directions for use as an undirected graph series.
We selected 14 days from 2nd Sunday of every month of 2014 for 348 bike stations, and constructed the sequence of adjacency matrices $\mathcal X \in \R^{348\times348\times336}$.
We consider that the relation between locations is stronger as the number of bikes increases.
In this experiment, to obtain the smooth adjacency matrix series to extract dynamic properties, we summed up 14 days of data for every month and perform 12-point (i.e., half day) moving average.
Figure \ref{fig:Bike}a shows an example of the preprocessed number of bikes between Lincoln memorial and three stations with a maximal number of bikes moving from/to Lincoln memorial.
These seemed to be coherent and cycled at daily and weekly cycles.

\begin{figure*}[h!]
\centering
\includegraphics[width=1\columnwidth]{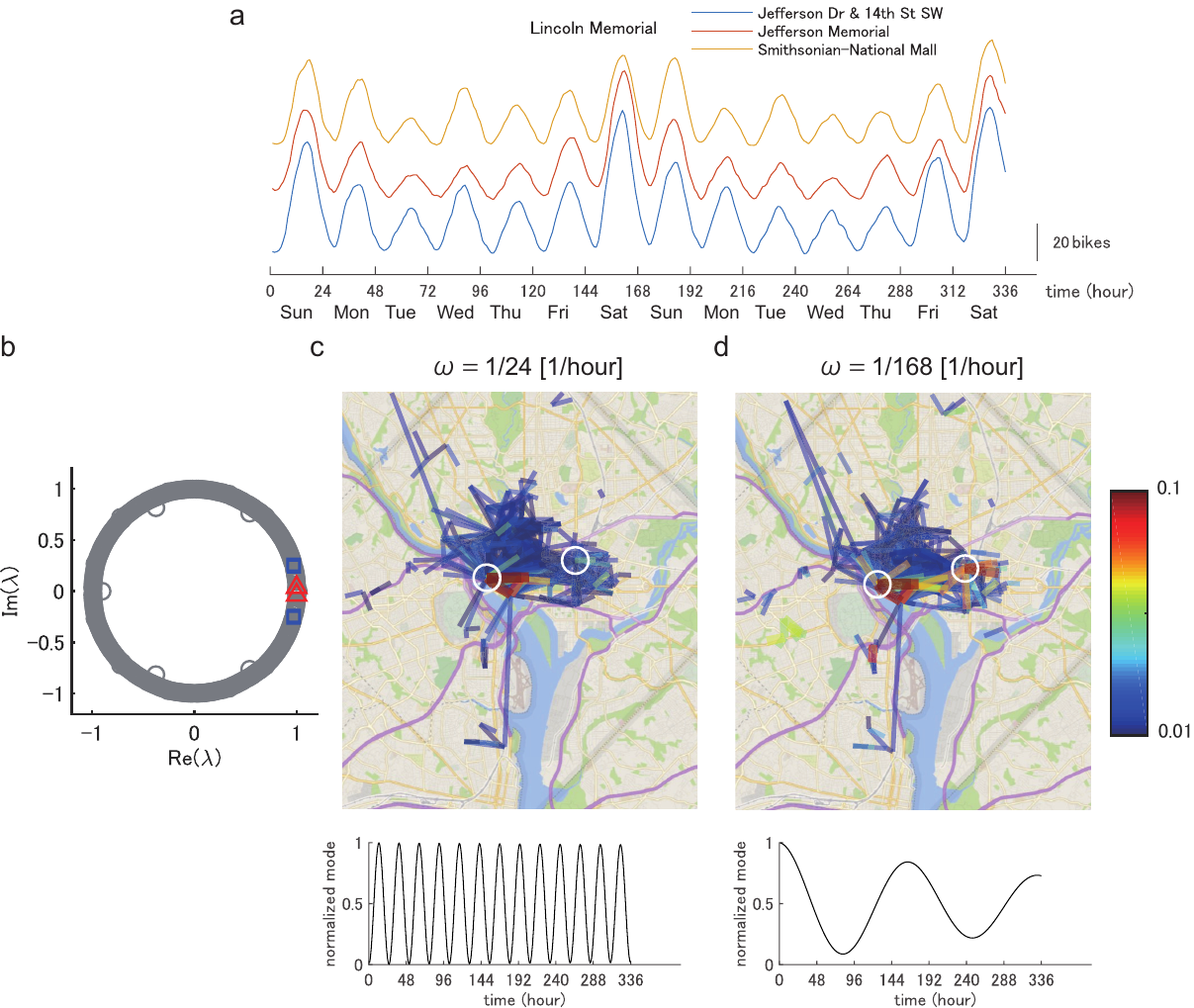}
\caption{Temporal and spatial modes of Graph DMD on the bike-sharing data of Washington D.C. 
(a) an example of the number of bikes between Lincoln memorial and three stations with a maximal number of bikes moving from/to Lincoln memorial. 
(b) a result of DMD eigenvalue. The blue square and red triangle indicate eigenvalues of approximately $\omega = 1/24$ and $\omega = 1/168$, respectively. (c) and (d): Graphical representation of the amplitude of spatial DMD modes in $\omega = 1/24$ and $\omega = 1/168$, respectively.
The left and right circles indicate the Lincoln Memorial and Columbus Circle / Union Station, respectively.
Lower time series in (c) and (d) are examples of the extracted temporal dynamics corresponding to the above spatial modes (visualized spatial modes are averaged among multiple modes).
For improving the visibility, we used only an eigenvalue and set initial values 0 and 1 for (c) and (d), respectively.
}
\label{fig:Bike}
\end{figure*}

Figure 2b shows the eigenvalues estimated by Graph DMD. %
We confirm that most eigenvalues are on the unit circle, indicating that the dynamics were almost oscillators. %
Among these eigenvalues, we focused on the specific temporal modes of the traffic such as daily and weekly periodicity~\citep{Takeuchi17}, i.e.,
we selected $\omega = |{\rm{Im}}({\rm{log}} (\lambda))|/\Delta t /(2\pi) = \{1/24, 1/168\}$, where $\lambda$ is Graph DMD eigenvalue ($\Delta t = 1$ [hour]).

Figure~\ref{fig:Bike}c and d shows the spatial (graph) pattern of the DMD mode for approximately $\omega = 1/24$ and $\omega = 1/168$ on the bike station map, respectively.
Although the bike transportation near Lincoln memorial (left circle in Figures \ref{fig:Bike}c and d) shows a stronger spectrum for both daily and weekly periodicity, the bike transportation in a downtown area near Union Station (right circle in Figures \ref{fig:Bike}c and d) for weekly periodicity shows a stronger spectrum than that for daily periodicity.

Overall, our method can extract the different spatial (graph) modes for specific temporal modes based on the dynamical structure.
Note that in the formulation (again, in functional space), we also assume that the covariance matrix $\bm{K}$ is a symmetric positive semi-definite matrix-valued function.
Numerically, however, as is the case of real-world data, we did not assume that the (symmetric) adjacency matrix $\bm A_t$ is positive semi-definite. 
In Appendix B, we proposed the alternative to modify them to positive semi-definite matrices  (the results were similar to Figure \ref{fig:Bike}). 

\subsection{Fish-schooling model}
\label{ssec:fish}
Next, we evaluated our method using a example with unknown global dynamics, because in some real-world (especially biological) data, the true global spatiotemporal structure is sometimes unknown \citep{Fujii18a,Hojo18}.
For evaluation, here we used well-known collective motion models~\citep{Vicsek95} with simple local rules to generate multiple distinct group behavioral patterns (Figure~\ref{fig:Fish}a): swarm, torus, and parallel behavioral shapes. 
The detailed configuration and simulation of the experiments are described in Appendix C and D, respectively. 
We used Gaussian kernels to create the sequences of adjacency matrices using inter-agent distance (for details, see Appendix E) because the local rules were applied based on the distance. 
First, the results in the temporal DMD mode, interpolating the discrete frequency spectra, exhibit a relatively wide spectrum for the swarm (Figure~\ref{fig:Fish}b), a narrow spectrum for the torus (Figure~\ref{fig:Fish}e) and parallel (Figure~\ref{fig:Fish}g). 
Among these spectra, we focused on characteristic low- (0-2 Hz) and high-frequency (2-4 Hz) modes.
The spectra in the swarm (Figure~\ref{fig:Fish}c,d) and torus (Figure~\ref{fig:Fish}f) show relatively stronger spectra nearer individuals, compared with that in the parallel (Figure~\ref{fig:Fish}h).
Thus, our method can visualize the observed interaction behaviors. 

\begin{figure}[h!]
\centering
\includegraphics[width=1\columnwidth]{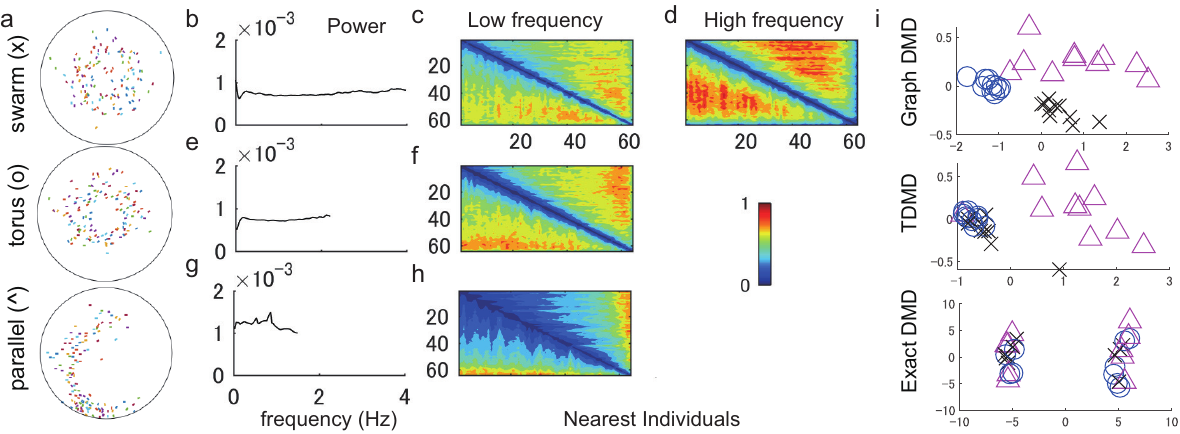} 
\vspace{-0mm}
\caption{Results with fish-schooling simulations. (a) Three different behavioral shapes. Temporal frequency (b,e,g) and spatial DMD spectra in low (c, f, h) and in high frequency mode (d) are shown. 
(i) Embedding with distance matrix of three methods. Symbols are given in (a).
}
\label{fig:Fish}
\vspace{-0mm}
\end{figure}

Although a direct and important application of Graph DMD is the extraction of the dynamical information for GDS, it can also perform embedding and recognition of GDSs using extracted features based on the dynamical structure.
For embedding the distance matrix with DMD modes such as using multidimensional scaling (MDS), the components of the distance matrix depend on the problem.
In this experiment, we compute the distance matrix between the temporal frequency modes by the alignment of the number of dimensions from larger frequencies, because of the results shown in Figure~\ref{fig:Fish}b,e,g. 
As comparable methods to extract dynamical information, we compared the result of our method with those of reformulated TDMD using the Cartesian coordinates and exact DMD breaking the tensor data structure (for details, see Appendix E).  
In Figure~\ref{fig:Fish}i, our method apparently distinguished the three types whereas reformulated TDMD and exact DMD did not.
We quantitatively evaluated the classification error with $k$-nearest neighbor classifier ($k = 3$) for simplicity.
We used 45 sequences in total and computed averaged 3-fold cross-validation error.
The classification error in Graph DMD (0.022) was smaller than those in TDMD (0.311) and exact DMD (0.511).

\vspace*{-0mm}
\section{Conclusions}
\label{sec:conclusion} 
\vspace*{-0mm}
In this paper, we formulated Koopman spectral analysis for NLDSs with structures among observable and proposed an estimation algorithm for performing it with a given sequence of data matrix with dependent structures among observables, which can be useful for understanding the latent global dynamics underlying such NLDSs from the available data.
To this end, we first formulated the problem of estimating the spectra of Koopman operator defined in vvRKHSs to incorporate the structure among observables, and then developed a procedure for applying this to the analysis of such NLDSs by reformulating Tensor-based DMD.
As a special case of our method, we proposed the method named as Graph DMD, which is a numerical algorithm for Koopman spectral analysis of graph dynamical systems, using a sequence of adjacency matrices.
We further considered applications using our method, which were empirically illustrated using both synthetic and real-world datasets.

\section*{Acknowledgments}
We would like to thank Naoya Takeishi, Isao Ishikawa, Masahiro Ikeda for the beneficial discussion. This work was supported by JSPS KAKENHI Grant Numbers 16K12995, 16H01548, 18K18116, and 18H03287, and AMED Grant Number JP18dm0307009. 



\bibliography{main_} 

\newpage

\setcounter{page}{1}
\appendix
\section*{Supplementary Materials}

\section{Relation between the projection and projected matrix in Section 3 and 4}
\label{relation}
In Section 3, we consider the orthogonal directions $\bm{\nu}_j = \sum_{t=0}^{\tau-1}\bm{\phi}_{\bm{c}}(\bm{x}_t)\alpha_{j,t} = \mathcal M_1 \bm{\alpha}_{j}$, where $\alpha_{j,t} \in \R$ and $\bm{\alpha}_j \in \R^{\tau}$.
Recall that we let ${\mathcal{U}} = [\bm{\nu}_1,\ldots,\bm{\nu}_p]$ and ${\mathcal{U}} = \mathcal M_1 \bm{\alpha}$ with the coefficient matrix $\bm{\alpha} \in \R^{\tau \times p}$. 
In Section 4, we mention that we need to compute the projected matrix $\bm M$ using the above projection function ${\mathcal{U}}$.
Here, we use the property of the feature map for time $t = 0,\ldots,\tau-1$: 
\begin{equation}
\label{eq:app1}
\left< \bm{\phi}_{\bm{c}}(\bm{x}_t), \bm{\phi}_{\bm{c}}(\bm{x}_t) \right>_{\bm{K}} = {\bm{c}_t}^T{\bm{K}(\bm{x}_t,\bm{x}_t)}{\bm{c}_t} = \vectorize(\bm{c}_t\bm{c}_t^T)^T \vectorize(\bm{K}(\bm{x}_t,\bm{x}_t)).
\end{equation}
If we multiply $\bm{\phi}_{\bm{c}}(\bm{x}_t)^*$ for $t = 0,\ldots,\tau-1$ from the left side of each column of $\mathcal{U}$, we obtain 
\begin{flalign}
&[\left< \bm{\phi}_{\bm{c}}(\bm{x}_0), \bm{\phi}_{\bm{c}}(\bm{x}_0) \right>_{\bm{K}},\ldots,\left< \bm{\phi}_{\bm{c}}(\bm{x}_{\tau-1}), \bm{\phi}_{\bm{c}}(\bm{x}_{\tau-1}) \right>_{\bm{K}}]\bm{\alpha}\label{eq:app2} \\ 
&= [\vectorize(\bm{c}_0\bm{c}_0^T)^T\vectorize(\bm{K}(\bm{x}_1,\bm{x}_1)),\ldots, \vectorize(\bm{c}_{\tau-1}\bm{c}_{\tau-1}^T)^T\vectorize(\bm{K}(\bm{x}_{\tau-1},\bm{x}_{\tau-1}))]\bm{\alpha}.
\nonumber
\end{flalign}
We consider the projected matrix mentioned in Section 3 as 

\noindent $ [\vectorize(\bm{K}(\bm{x}_0,\bm{x}_0)),\ldots, \vectorize(\bm{K}(\bm{x}_{\tau-1},\bm{x}_{\tau-1}))]\bm{\alpha}$.
In Section 4.2, we obtain $\bm{M} = \bm{X}\bm{N}^\dagger \bm{\Sigma}^{-1} $ as a projected matrix for Graph DMD computation. 
Recall that we regard the observed $\mathcal A_{:,:,t} = \bm{A}_t$ as the realization of the structure of $\bm{K}(\bm{x}_t,\bm{x}_t)$ and we matricize $\mathcal X = \mathcal A_{:,:,0:\tau-1} $ to obtain $\bm{X} = [\vectorize(\bm{A}_0),\ldots,\vectorize(\bm{A}_{\tau-1})]$.
Then, $\bm{N}^\dagger \bm{\Sigma}^{-1} $ in Section 4 corresponds $\bm{\alpha}$ in Section 3.
Therefore, we can show the relation between a set of orthogonal projection $\mathcal U$ in Section 3 and the projected matrix $\bm{M}$ in Section 4.

\newpage
\section{Modification to positive semi-definite matrices for bike sharing data}
\label{SetupBike}
If we consider the case that the sequence of adjacency matrices satisfies the positive semidefiniteness, since in this setting we cannot use any scalar-valued (positive semidefinite) kernels (i.e., we cannot define the kernel), we can modify
the sequence of adjacency matrix by adding diagonal matrix of minimum eigenvalues $\lambda_{min}$ such that $\bm{A}_t' = \bm{A}_t + \lambda_{min} \bm{I}$ to satisfy the positive semidefiniteness and to reflect the information of number of bikes.
This procedure would be mathematically reasonable, because 
$\bm{A}_t' \bm{u}_i = \bm{A}_t \bm{u}_i+\lambda_{min}\bm{u}_i = \lambda_i\bm{u}_i+\lambda_{min}\bm{u}_i = (\lambda_i+\lambda_{min})\bm{u}_i$, where $\lambda_{i}$ and $\bm{u}_i$ are $i$th eigenvalue and eigenvector, respectively.
Qualitatively, the results in Figure \ref{fig:BikeS1} were similar to Figure 2.

\begin{figure*}[h!]
\centering
\includegraphics[width=1\columnwidth]{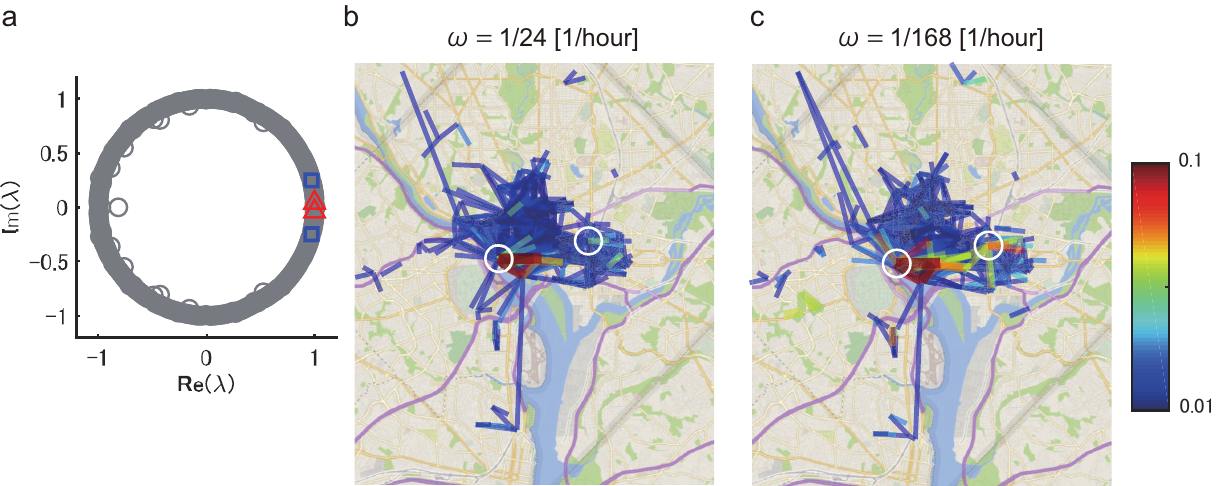}
\caption{Temporal and spatial modes of Graph DMD on the bike-sharing data of Washington D.C. 
The difference from Figure 2 is that the adjacency matrices were modified to positive semi-definite matrices.
Configurations in (a)-(c) are the same as Figure 2 (b)-(d).
Qualitatively, the results were similar to Figure 2.
}
\vspace{-4mm}
\label{fig:BikeS1}
\end{figure*}

\newpage
\section{Configuration of the fish-schooling model}
\label{ConfigFish}
As a multiagent model, individual-based models that simulate swarming, torus-like, or parallel group behaviors of  fish-schooling~\citep{Couzin02} are good examples because the relation between the properties of the local system and the emergence of global behavior are well-known and explicit. 

The schooling model we used in this study was a unit-vector based (rule-based) model, which accounts for the relative positions and direction vectors neighboring fish agents, such that each fish tends to align its own direction vector with those of its neighbors. We used an existing model based on the previous work~\citep{Couzin02}. 
In this model, 64 agents (length: 0.5 m) are described by a two-dimensional vector with a constant velocity (4 m/s) in a boundary circle (radius: 25 m) as follows: 
${\bm{r}}_i=\left({x_i}~{y_i}\right)^T$ and ${\bm{v}}_i\left(t\right)= \|\bm{v}_i\|_2\bm{d}_i$, where $x_i$ and $y_i$ are two-dimensional Cartesian coordinates, ${\bm{v}}_i$ is a velocity vector, $\|\cdot\|_2$ is Euclid norm, and $\bm{d}_i$ is an unit directional vector for agant $i$.

At each time step, a member will change direction according to the positions of all other members. The space around an individual is divided into three zones with each modifying the unit vector of the velocity.
The first region, called the repulsion zone with radius $r_r = 1$ m, corresponds to the ``personal'' space of the particle. Individuals within each other’s repulsion zones will try to avoid each other by swimming in opposite directions. 
The second region is called the orientation zone, in which members try to move in the same direction  (radius $r_o$). 
We changed the parameter $r_o$ to generate the three behavioral shapes (we set $r_o$ to $2$, $10$, and $13$: see Appendix D). 
Next is the attractive zone (radius $r_a = 15 m$), in which agents swim towards each other and tend to cluster, while any agents beyond that the radius has no influence. 
Let $\lambda_r$, $\lambda_o$, and $\lambda_a$ be the numbers in the zones of repulsion, orientation and attraction respectively. For $\lambda_r \neq 0$, the unit vector of an individual at each time step $\tau$ is given by:
\begin{equation} 
\bm{d}_i(t+\tau , \lambda_r \neq 0 )=-\left(\frac{1}{\lambda_r-1}\sum_{j\neq i}^{\lambda_r}\frac{\bm{r}_{ij}(t)}{\|\bm{r}_{ij}(t)\|_2}\right),
\end{equation} 
where ${\bm{r}}_{ij}={\bm{r}}_j-{\bm{r}}_i$.
The velocity vector points away from neighbors within this zone to prevent collisions. This zone is given the highest priority; if and only if $\lambda_r = 0$, the remaining zones are considered. 
The unit vector in this case is given by:
\begin{equation} 
\bm{d}_i(t+\tau , {\lambda}_r=0)=\frac{1}{2}\left(\frac{1}{\lambda_o}\sum_{j=1}^{\lambda_o} \bm{d}_j\left(t\right)+\frac{1}{\lambda_a-1}\sum_{j\neq i}^{{\lambda}_a}\frac{{\bm{r}}_{ij}\left(t\right)}{\|{\bm{r}}_{ij}\left(t\right)\|_2}\right).
\end{equation} 
The first term corresponds to the orientation zone while the second term corresponds to the attraction zone. The above equation contains a factor of $1/2$ which normalizes the unit vector in the case where both zones have non-zero neighbors. If no agents are found near any zone, the individual maintains constant velocity at each time step.

In addition to the above, we constrain the angle by which a member can change its unit vector at each time step to a maximum of $\beta = 30$ deg. This condition was imposed to facilitate rigid body dynamics. Because we assumed point-like members, all information about the physical dimensions of the actual fish is lost, which leaves the unit vector free to rotate at any angle. In reality, however, conservation of angular momentum will limit the ability of the fish to turn angle $\theta$ as follows:
\begin{equation} 
  \bm{d}_i\left(t+\tau \right)\cdot \bm{d}_i\left(t\right) = 
  \begin{cases}
   \cos(\beta ) & \text{if $\theta >\beta$} \\
   \cos\left(\theta \right) & \text{otherwise}.
  \end{cases}
\end{equation} 
If the above condition is not met, the angle of the desired direction at the next time step is rescaled to $\theta = \beta$. In this way, any un-physical behavior such as having a 180$^\circ$ rotation of the velocity vector in a single time step, is prevented.

\section{Simulation of fish-schooling model}
The initial conditions were set such that the particles would generate a torus motion, though all three motions emerge from the same initial conditions. The initial positions of the particles were arranged using a uniformly random number on a circle with a uniformly random radius between 6 and 16 m (the original point is the center of the circle). Boundary radius was set to 25 m. The initial velocity was set to be perpendicular to the initial position vector. The average values of the control parameter $r_o$ were set to 2, 10, and 13 to generate the swarm, torus, and parallel behavioral shapes, respectively. We simply added noise to the constant velocities among the agents (but constant within a particle) with a standard deviation of $\sigma = 0.05$. The time step in the simulation was set to $10^{-2}$ s. We simulated ten trials for each parameter $r_o$ in 10 s intervals (1000 frames). The analysis start times were varied depending on the behavior type to avoid calculating the transition period (torus: 10 s, swarm, and parallel: 30 s after the simulation start). 
\label{SimFish}

\section{Embedding and recognition of fish-schooling dynamics}
\label{EmbedFish}
To create the sequence of adjacency matrices for Graph DMD, we used Gaussian kernels.
Let $\bm{z}_{i,t}$ and $\bm{z}_{j,t}$ be two-dimensional Cartesian coordinates for particle $i$ and $j$ at time $t$.
Then, Gaussian kernel as the component of the adjacency matrix is given as 
\begin{equation} 
  \bm{A}_{i,j,t} = \exp\left(\frac{-\|\bm{z}_{i,t}-\bm{z}_{j,t}\|^2_2}{2\sigma'}\right),
\end{equation} 
where $\sigma'$ is set to $25^2/2\rm{log}2$ such that $\bm{A}_{i,j,t} = 0.5$ if $\|\bm{z}_{i,t}-\bm{z}_{j,t}\|_2$ is equivalent to the boundary radius (25 m).
Since the order of the adjacency matrix is not uniquely determined, we sorted it by the nearest-neighbors.
We simulated 64 agents in 1000 frame intervals; thus the size of the sequence of adjacency matrices are $\mathcal A \in \R^{64\times64\times1000}$.

For embedding the distance matrix with Graph DMD modes using MDS, we compute the distance matrix between the temporal frequency modes.
Note that if there are similar frequencies or interactions between temporal and spatial DMD modes, the spectral kernel~\citep{Fujii17} can be effective in the case with the matrix form.
However, we need to define the spectral kernel in tensor form for this application. 

For the comparable methods to extract frequency information, we also perform reformulated TDMD with the tensor data $\mathcal X \in \R^{64\times2\times1000}$ using two-dimensional Cartesian coordinates and basic DMD with the data matrix $\bm X \in \R^{64^2\times1000}$ breaking the tensor structure of the Euclid distance matrix series.
\end{document}